\documentclass[10pt]{article} 
\usepackage[accepted]{rlc}

\usepackage{graphicx}           
\usepackage{subcaption}         
\usepackage[space]{grffile}     
\usepackage{url}                

\usepackage{amsmath}
\usepackage{amssymb}
\usepackage{mathtools}
\usepackage{amsthm}
\usepackage{bm}
\usepackage{thm-restate}
\usepackage{mathrsfs}           
\mathtoolsset{showonlyrefs}     

\theoremstyle{plain}
\newtheorem{theorem}{Theorem}[section]

\newtheorem{lemma}[theorem]{Lemma}

\theoremstyle{definition}

\newtheorem{assumption}[theorem]{Assumption}
\theoremstyle{remark}

\usepackage{algorithm}
\let\classAND\AND
\let\AND\relax
\usepackage{algorithmic}

\let\AND\classAND
\AtBeginEnvironment{algorithmic}{\let\AND\algoAND}

\usepackage{enumitem}

\newcommand{\myeq}[1]{\ensuremath{\stackrel{\textit{\text{#1}}}{=}}}
\newcommand{\myleq}[1]{\ensuremath{\stackrel{\textit{\text{#1}}}{\leq}}}

\newcommand{\mypropto}[1]{\ensuremath{\stackrel{\textit{\text{#1}}}{\propto}}}

\newcommand{\algname}{Options-UCBVI}

\DeclareMathOperator*{\E}{\mathbb{E}}
\DeclareMathOperator*{\Var}{\mathbb{V}\mathrm{ar}}
\DeclareMathOperator{\V}{\mathbb{V}}

\newcommand{\argmax}{\text{argmax}}

\usepackage{thmtools} 
\usepackage{thm-restate}
\usepackage{bbm}

\title{A Provably Efficient Option-Based Algorithm\\for both
High-Level and Low-Level Learning}

\author{Gianluca Drappo\\
    \texttt{gianluca.drappo@polimi.it}\\
    DEIB\\
    Politecnico di Milano\\ 
    Milan, 20133, Italy\\
    \And
    Alberto Maria Metelli\\
    \texttt{albertomaria.metelli@polimi.it}\\
    DEIB\\
    Politecnico di Milano\\ 
    Milan, 20133, Italy\\
    \And
    Marcello Restelli\\
    \texttt{marcello.restelli@polimi.it}\\
    DEIB \\
    Politecnico di Milano\\ 
    Milan, 20133, Italy
    }

\allowdisplaybreaks[4]
\begin{document}
\setlength{\abovedisplayskip}{3pt}
\setlength{\belowdisplayskip}{3pt}
\setlength{\textfloatsep}{5pt}

\maketitle

\begin{abstract}
Hierarchical Reinforcement Learning (HRL) approaches have shown successful results in solving a large variety of complex, structured, long-horizon problems. Nevertheless, a full theoretical understanding of this empirical evidence is currently missing. In the context of the \emph{option} framework, prior research has devised efficient algorithms for scenarios where options are \emph{fixed}, and the high-level policy selecting among options only has to be learned. However, the fully realistic scenario in which \emph{both} the high-level and the low-level policies are learned is surprisingly disregarded from a theoretical perspective. This work makes a step towards the understanding of this latter scenario. Focusing on the finite-horizon problem, we present a meta-algorithm alternating between regret minimization algorithms instanced at different (high and low) temporal abstractions. At the higher level, we treat the problem as a Semi-Markov Decision Process (SMDP), with fixed low-level policies, while at a lower level, inner option policies are learned with a fixed high-level policy.
The bounds derived are compared with the lower bound for non-hierarchical finite-horizon problems, allowing to characterize when a hierarchical approach is provably preferable, even without pre-trained options.
\end{abstract}

\section{Introduction}
Hierarchical Reinforcement Learning \citep[HRL,][]{pateria2021hierarchical} is a framework in the class of Reinforcement Learning~\citep[RL,][]{sutton2018reinforcement} methods that has shown successful results in recent years thanks to its ability to deal with complex, long-horizon, and structured problems \citep{bacon2017option, vezhnevets2017feudal, levy2019learning, nachum2018data}. In a large variety of real-world scenarios, a complex task can be decomposed as a concatenation of different sub-tasks that are often solved as a whole to learn the optimal policy. Nevertheless, in several cases, these sub-tasks are not fully coupled, and solving them separately leads to (near)optimal solutions. In these circumstances, a \emph{hierarchical} RL approach could deliver significant benefits w.r.t. the application of \emph{flat} RL algorithms, thanks to its ability to properly exploit the structure of the environment.
A common example in the HRL literature \citep{dietterich2000hierarchical} is the \emph{taxi problem}, in which an autonomous agent controls a taxi that has to bring a passenger from a starting point to a destination location. This problem embodies three different tasks: (i) driving, (ii) picking up, and (iii) dropping off the passenger. The HRL power resides in the explicit exploitation of this inner structure, subdividing the problem into a set of sub-tasks, individually solvable with their own optimal policies, which are then linked sequentially, one after the other.
This approach naturally reduces each problem's complexity, letting the agent focus on one objective at a time. 

Recent works have attempted to analyze the theoretical benefits that motivate the great successes of HRL in practice \citep{mann2015approximate, fruit2017exploration, fruit2017regret, wen2020efficiency, drappo2023an, robert2024sample}. Most of them focus on problems organized in two-level hierarchies, where the high-level policy has control over a set of \emph{pre-trained} \emph{options}~\citep{precup1997multi}, i.e., a particular formalization of temporally extended actions or sub-tasks, and the options' policies control the actual interaction with the environment throughout the \emph{primitive actions}. Using this set of fixed options helps to reduce the complexity of particular classes of problems, where the structure enforced by the options does not compromise optimality \citep{fruit2017exploration, fruit2017regret}. 
While this clearly motivates the performance improvements empirically experienced in several tasks, when to prefer such approaches in situations where \emph{no pre-trained} supportive policies are available, and, thus, the agent is required to face the problem from scratch, solving both the high and the low-level training, is still an open question. 
To the best of our knowledge, only \cite{drappo2023an} provide a preliminary insight in this direction, proposing an approach that first learns the optimal options' policies and then exploits them to learn the original task. However, while overcoming the need for a fixed set of pre-trained policies, they incur sub-optimal performances as any Explore-then-Commit approach \citep{lattimore2020bandit}, making it hardly comparable with the best performance achievable by a flat algorithm.\footnote{An extended discussion of the related works can be found in Appendix \ref{app:relworks}.}

This paper aims to introduce High-Level/Low-level Meta-Learning, the first method designed to efficiently handle the lack of pre-trained policies, enabling effective learning of the entire task from scratch. The key idea involves dividing the learning process of the two levels into multiple phases, rather than just two, and consistently switching between them by keeping one level fixed while the other is learning. In this way, the inherent non-stationarity that arises is mitigated. However, to have efficient performances, a fundamental requirement is the use of efficient regret minimizers for both levels. Nevertheless, while \cite{azar2017minimax} proposed an algorithm that achieves the best possible performance in FH-MDPs (i.e., the \emph{low-level}), no existing works in the literature propose a valid alternative when dealing with temporally extended actions. Therefore, to jointly learn both level policies, we introduce \algname, an efficient regret minimizer based on UCBVI for FH-SMDPs, to handle the \emph{high-level} problem efficiently.

\textbf{Original Contributions}~~The contributions of this paper can be summarized as follows:
\begin{itemize}[noitemsep, leftmargin=*, topsep=0pt]
    \item We derive \emph{\algname} (O-UCBVI), a novel regret minimization algorithm for FH-SMDPs, that enjoys an upper bound on the regret of order $\tilde{O}(H\sqrt{SOKd})$\footnote{$\Tilde{O}$ neglects logarithmic terms.}, where $S$ the number of state, $O$ the cardinality of the option set given, $d$ the average per-episode number of played options, and $K$ the number of episodes (Section \ref{sec:opt-ucbvi}).
    \item We propose the first algorithm, named \emph{High-Level/Low-level Meta-Learning} (HLML), for simultaneously learning at both the high- and the low-levels, exploiting \algname~for the \emph{high-level} and UCBVI for the \emph{low-level} (i.e., the options learning). It provides regret guarantees of order $\Tilde{O}(C^LH\sqrt{SOKd}+ C^H H_O\sqrt{OSAKH_O})$ where other than the already mentioned constants, $A$ is the primitive action space cardinality, $C^H$, and $C^L$ are concentrability coefficient that will be analyzed later, and $H_O$ is an upper bound of the options' duration. By comparing this result with the lower bound on the regret for \emph{flat} problems \citep{osband2016lower}, we've been able to characterize specific classes of problems in which the former delivers provably better theoretical guarantees, answering the question \textit{``when to prefer HRL to standard RL, if both high-level and low-level policies are unknown?''}(Section \ref{sec:two-phase}).
\end{itemize}
The proofs of all the results presented in the main paper are reported in the Appendix~\ref{proof1}-\ref{proof2}.

\section{Problem Formulation}\label{sec:preliminaries}
In this section, we provide the necessary background employed in the subsequent sections.\footnote{Let $N \in \mathbb{N}$, we denote with $[N] \coloneqq \{1,\dots,N\}$.}

\textbf{Finite-Horizon MDPs}~~A Finite-Horizon Markov Decision Process \citep[FH-MDP,][]{puterman2014markov} is a tuple $\mathcal{M} = (\mathcal{S, A}, r^L, p^L, H)$, where $\mathcal{S}$ is the state space with caridnality $S$; $\mathcal{A}$ the (\emph{low-level} or \emph{primitive}) action space with cardinality $A$; $r^L: \mathcal{S\times A} \times [H] \rightarrow [0,1]$ is the reward function, which quantifies the quality $r^L(s,a,h)$ of action $a\in\mathcal{A}$ in state $s\in \mathcal{S}$ at stage $h\in [H]$; $p^L: \mathcal{S\times A\times} [H] \times \mathcal{S} \rightarrow [0,1]$ is the transition model, defining the probability $p^L(s'|s,a,h)$ of transitioning to state $s' \in \mathcal{S}$ by taking action $a\in \mathcal{A}$ in state $s\in\mathcal{S}$ at stage $h \in [H]$; and $H \in \mathbb{N}$ is the horizon. 
The behavior of an agent is modeled by a (\emph{low-level}) deterministic policy $\pi: \mathcal{S} \times [H] \rightarrow \mathcal{A}$ that maps a state $s \in \mathcal{S}$ and a stage $h \in [H]$ to a (\emph{low-level} or \emph{primitive}) action $\pi(s,h) \in \mathcal{A}$. 

\textbf{Finite-Horizon Semi-MDPs}~~A Finite-Horizon Semi-Markov Decision Process \citep[FH-SMDP,][]{drappo2023an} is the adaptation of Semi-Markov Decision Processes \citep[SMDP]{baykal2010semi} to finite-horizon setting. An FH-SMDP is defined as a tuple $\mathcal{SM} = (\mathcal{S, O}, r^H, p^H, H)$, where $\mathcal{S}$ and $H$ are the same quantities of FH-MDPs;
$\mathcal{O}$ is a set of temporally extended actions (\emph{high-level}), with cardinality $O$; $r^H: \mathcal{S} \times \mathcal{O} \times [H] \rightarrow [0,H]$ is the (\emph{high-level}) cumulative reward obtained $r^H(s,o,h)$, until the temporally extended (\emph{high-level)} action $o \in \mathcal{O}$ terminates, when selected in state $s \in \mathcal{S}$, at stage $h \in [H]$; $p^H: \mathcal{S} \times \mathcal{O} \times [H] \times \mathcal{S} \times [H] \rightarrow [0,1]$ is the transition model, defining the probability $p^H(s',h'|s,o,h)$ of transitioning to state $s' \in \mathcal{S}$, after $(h-h')$ time steps, $h' \in [H]$, when playing (\emph{high-level}) action $o \in \mathcal{O}$, in state $s \in \mathcal{S}$, and stage $h \in [H]$.
The behavior of an agent is modeled by a deterministic (\emph{high-level}) policy $\mu: \mathcal{S} \times [H] \rightarrow \mathcal{O}$ that maps a state and a stage $h \in [H]$ to a (\emph{high-level}) action $\mu(s,h) \in \mathcal{O}$.

HRL builds upon the theory of Semi-MDPs, characterizing the concept of temporally extended action with fundamentally two frameworks \citep{pateria2021hierarchical}: sub-tasks \citep{dietterich2000hierarchical} and options \citep{sutton1999between}. For the sake of this paper, we focus on the options framework. 

\textbf{Options}~~An option \citep{sutton1999between} is a temporally extended action characterized by three components $o = (\mathcal{I}^o, \beta^o, \pi^o)$. $\mathcal{I}^o \subseteq \mathcal{S} \times [H]$ is the subset of states and stages pairs $(s,h)\in \mathcal{S} \times [H]$ in which the option can start, $\beta^o: \mathcal{S} \times [H] \to [0,1]$ defines the probability $\beta^o(s,h)$ that an option terminates in state $s \in \mathcal{S}$ and stage $h \in [H]$, and, $\pi^o: \mathcal{S} \times [H] \to \mathcal{A}$ is the deterministic policy executed once an option is selected and until its termination.

Before proceeding, we introduce the following standard assumption.
\begin{assumption}[Admissible options \cite{fruit2017exploration}]
The set of options $\mathcal{O}$ is assumed \emph{admissible}, i.e., $\forall o \in \mathcal{O},\, s \in \mathcal{S},\, \text{and } h \in [H]: \ \beta^o(s,h) > 0 \implies \exists o' \in \mathcal{O}: \ (s,h) \in \mathcal{I}^{o'}$. 
\end{assumption}
The assumption is a minimal requirement for the problem to be well-defined, and it guarantees that whenever an option $o$ stops in a state $s$ at stage $h$, there always exists another option $o'$ that can start from the state-stage pair $(s,h)$.

\textbf{Average per-episode duration}~~ In the following analysis, we will refer to $d$ \citep{drappo2023an} as the average per-episode number of decisions taken in an episode of length $H$:
\begin{align}
\resizebox{.33\textwidth}{!}{$\displaystyle
    d \coloneqq \frac{1}{K} \sum_{o \in \mathcal{O}} \sum_{s \in \mathcal{S}} \sum_{h \in [H]} n_{K+1}(s,o,h) \label{eq:d}
$}
\end{align}
where $n_{K+1}(s,o,h)$ is the number of times a temporally extended action (or option) $o$ has been selected in state $s$, in step $h$, up to episode $K$ of interaction with the environment.

\textbf{Problem Formulation}~~We are given a set of \emph{not pre-trained} options $\mathcal{O}$, i.e., for every option $o \in \mathcal{O}$, the initiation set $\mathcal{I}^{o}$ and the termination function $\beta^o$ are \emph{fixed}, while the inner low-level policy $\pi^o$ has to be learned. We seek to learn \emph{both} the high-level policy $\mu$ (selecting options in the FH-SMDP) and the low-level policies  $\pi^o$ (inner to the options) for every $o \in \mathcal{O}$ as follows:
\begin{align}\label{eq:eqeqeq}
    (\mu^*, \bm{\pi}^*) \in \mathop{\mathrm{argmax}}_{\mu,\bm{\pi}} V^{\mu}_{\bm{\pi}}(s_1,1),
\end{align}
where $\bm{\pi} = (\pi^o)_{o \in \mathcal{O}}$ are the low-level policies and $\mu$ is the high-level policy, $s_1 \in \mathcal{S}$ is an initial state, and $V_{\bm{\pi}}^{\mu}$ is the value function, defined for every $(s,h) \in \mathcal{S}\times [H]$ as:
\begin{align}
    &\resizebox{.6\textwidth}{!}{$\displaystyle
    V_{\bm{\pi}}^{\mu}(s,h) \coloneqq \E_{(s',h')\sim p^H(\cdot |s,\mu(s,h),h)}\Big[ r^H(s,\mu(s,h),h) + V^{\mu}_{\bm{\pi}}(s', h') \Bigr],$} \\
    &\resizebox{.8\textwidth}{!}{$\displaystyle
    r^H(s,o,h) \coloneqq \E_{s'' \sim p^L(\cdot|, s, \pi^o(s,h), h)} \Big[ r^L(s,\pi^o(s, h),h) + (1 - \beta^o(s'',h+1))r^H(s'', o, h+1) \Big].$}
\end{align}
We denote with $V^*_{\bm{*}}(s_1,1) = V^{\mu^*}_{\bm{\pi}^*}(s_1,1)$. 

\textbf{Regret}~~The \emph{(cumulative) regret}~\citep{azar2017minimax, fruit2017exploration, zanette2018, drappo2023an} of an algorithm $\mathfrak{A}$ for the problem defined above is the cumulative value difference over $K$ episodes when playing the high-level policy $\mu_k$ and the low-level policies $\bm{\pi}_k$ at the episode $k \in [K]\coloneqq \{1,\dots,K\}$ instead of the optimal ones:
\begin{align*}
    R(\mathfrak{A}, K) \coloneqq \sum_{k=1}^K V^*_{\bm{*}}(s_1,1) - V^{\mu_k}_{\bm{\pi}_k}(s_1,1)
\end{align*}
Thus the goal of the algorithm is to play a sequence of policies $\mu_0,  \dots, \mu_K$, and $\bm{\pi}_0,  \dots, \bm{\pi}_K,$ such that R($\mathfrak{A},K$) is as small as possible.

\section{\algname}\label{sec:opt-ucbvi}
In order to develop an algorithm that jointly solves both high and low-level problems, it is necessary to handle each level efficiently. However, while \cite{azar2017minimax} proposes an optimal method for FH-MDP to solve the low level, no provably efficient counterparts have been proposed for FH-SMDPs, leaving the high level untreated. In this section, we introduce the first novel contribution of this work, which is a provably efficient algorithm for this framework.

Our method, named \emph{\algname} (O-UCBVI, Algorithm \ref{alg:opt-ucbvi}), is a model-based approach built upon UCBVI \citep{azar2017minimax} that exploits the given set of options $\mathcal{O}$ to learn the optimal FH-SMDP policy $\mu^*$. The key contribution of this algorithm is its explicit handling of temporally extended actions, which introduces an additional source of stochasticity due to their random duration. To address this issue, first, an estimate of the transition model is computed solely with the data collected from the SMDP, generating an estimate of a \emph{multi-step dynamic}, thereby ignoring the primitive state-action pairs visited during option execution (line \ref{alg-row:p}). Then, we address a more crucial point: the random duration of options (i.e., temporally extended actions) makes the strict application of backward induction, used by UCBVI to compute the optimistic value function, unfeasible. Intuitively, the value of a certain state-step pair, $V(s, h)$, needs to be back-propagated not only to the previous state-step pair $(s_{h-1}, h-1)$ but to any state-step pair where an option that would ultimately lead to $(s, h)$ could be selected. To handle this problem, we introduce a \emph{backward-forward} mechanism presented in lines \ref{alg-row:backward_b}-\ref{alg-row:backward_e}. Within the first loop, $h = H,...,1$, we move \emph{backward}, as in standard backward induction. However, in the inner loop, $h' = h+1,..., H+1$, we project to any possible future state-step pair reachable by playing an option in the current one (\emph{forward move}) to update the current value with those of future pairs. By employing this backward induction, we handle the randomness of the options' duration, ensuring proper computation of the values.

Up to this crucial change, \algname~ follows the same philosophy as UCBVI-BF. It implements the concept of \emph{optimism in the face of uncertainty} for SMDPs, with a tailored bonus added to the empirical Bellman operator (lines \ref{alg-row:bonus}-\ref{alg-row:bellman}), which mitigates the exploitation of known solutions and encourages strategic exploration of more uncertain regions of the SMDP. From a technical perspective, we modified the exploration bonus to deal with the non-stationary transition model and the set of given options, with their temporally extended nature. In particular, we focused on the version using the \emph{Bernstein-Freedman} \citep{freedman1975tail, maurer2009empirical} bonus in order to achieve tight regret guarantees.
Therefore, by following the same intuition of the analysis of UCBVI and adapting it to non-stationary transitions and the different backward induction, we end up demonstrating the following regret guarantee.

\begin{restatable}{theorem}{optucbvi}\label{thm:opt-ucbvi}
Let $\mathcal{SM}$ be an FH-SMDP with $S$ states and $O$ temporally extended actions (options), known reward,\footnote{The choice of assuming a known reward is for compliance with~\cite{azar2017minimax}. Nevertheless, learning the reward function is known to be a negligible task compared to learning the transition model of the environment and, consequently, will not alter the regret order.} bounded primitive reward $r^L(s,a,h) \in [0,1]$. The regret suffered by algorithm \algname~in $K$ episodes of horizon $H$ is bounded, with probability $1-\delta$, by:
\begin{align} \label{eq:opt-regret}
    \text{Regret}(\text{O-UCBVI}, K) \leq \Tilde{O} \left(H\sqrt{SOKd} + H^3S^2Od + H\sqrt{Kd} \right),
\end{align}
where $d$ is the average per-episode number of options played during the execution of the algorithm.
\end{restatable}
That, for $K \geq H^4S^3Od$ translates into a regret bound of $\Tilde{O}(H\sqrt{SOKd})$. 

The regret of this algorithm differs from the regret of UCBVI, $\Tilde{O}(H\sqrt{SAKH})$\footnote{The result in \cite{azar2017minimax} doesn't present the additional $\sqrt{H}$ term, which however is well-known to be tight even in standard FH-MDPs when the transition model is non-stationary. The non-stationarity of the transition model is unavoidable in the Semi-Markov setting due to the different durations of the temporally extended actions.}, for the term $\sqrt{O}$ replacing $\sqrt{A}$, which is the options set cardinality, and for the key term $\sqrt{d}$, instead of $\sqrt{H}$, which is the average per-episode number of options selected in $H$ steps.\\
This last term expresses the actual power endorsed by the options that allow a faster and wider exploration of the problem space and reduce the \emph{effective planning horizon}. Indeed, this is visible from the regret that scales with $\sqrt{OKd}$ instead of $\sqrt{AKH}$ as in the \emph{flat} version, and since $d \ll H$, and normally $O \leq A$, being the options longer and often fewer than primitive actions, \algname~suffers smaller regret than its flat counterpart when fixed options are given. In addition, we can show how this result is a generalization of the flat case. The upper bound is tight in its dominating term also when considering $\mathcal{O} = \mathcal{A}$ and, consequently, $d = H$, i.e., running \algname~on the flat MDP.

Now, given an optimal method for the high-level problem (i.e., tight in all the dependencies), we are ready to present the algorithm that jointly learns both level policies.

\begin{algorithm}[t]
\caption{\algname} \label{alg:opt-ucbvi}
\small
\begin{algorithmic}[1]
    \STATE{\textbf{Input:} $\mathcal{S, O}$, $H$, $K$}
    \STATE{Initialize $\mu_0$ arbitrarily, $Q_1(s,o,h) = 0$ for all $(s,o,h) \in \mathcal{S} \times  \mathcal{O} \times [H]$,  $L = \log(5SOKH/\delta)$, $\mathcal{D}^H \leftarrow \{\}$}
    \vspace{0.1cm}
    \FOR{$k=1, \dots, K$}
        \STATE{Compute $n_k(s,o,h) = \sum_{(x, y, z) \in \mathcal{D}^H} \mathbbm{1}\{x=s, y=o, z=h\}$}
        \STATE{Estimate $\hat{P}_k(s',h'|s,o,h) = \frac{1}{\max\{1,n_k(s,o,h)\}} \sum_{(x, y, z, w, u) \in \mathcal{D}^H} \mathbbm{1}\{(x, y, z, w, u) = (s,o,h,s',h')\} $} \label{alg-row:p}
        \STATE{Set $Q_k(s,o,H+1) = 0$ for all $(s,o,h) \in \mathcal{S} \times \mathcal{O} \times [H]$}
        \FOR{$h = H, \dots, 1$} \label{alg-row:backward_b}
            \FOR{$(s,o) \in \mathcal{S \times O}$}
                \FOR{$h' = h+1, \dots, H+1$}
                    \STATE{$\quad\Tilde{V}^{\mu_k}(s,h')= \min\left\{H-(h'-1), \max_{o \in \mathcal{O}}Q_{k}(s,o,h')\right\}$}
                \ENDFOR
                \STATE{$b_{hk}(s,o) = \sqrt{\frac{8L\Var_{(s',h') \sim \hat{P}_k}[\Tilde{V}^{\mu_k}(s',h')]}{n_k(s,o,h)}} + \frac{14HL}{3n_k(s,o,h)}+\sqrt{\frac{8\sum_{(s',h')}\hat{P}_k(s', h'|s,o,h)\min\big\{\frac{100^2H^5S^2OL^2}{\sum_{o}n_k(s',o,h')}, H^2\big\}}{n_k(s,o,h)}}$} \label{alg-row:bonus} 
                \STATE{$Q_k(s,o,h)= r(s,o,h) + \sum_{(s',h')}\hat{P}_k(s'h'|s,o,h)\Tilde{V}^{\mu_k}(s',h') + b_{hk}(s,o)$}\label{alg-row:bellman}
            \ENDFOR
        \ENDFOR \label{alg-row:backward_e}
        \STATE{$\mu_k(s,h) = \argmax_{o \in O} Q_k(s,o,h)$}
        \STATE{$s \leftarrow s_1$}
        \WHILE{$h < H$}
            \STATE{Play option $o = \mu_k(s,h)$, observe $(s',h')$, and update $\mathcal{D}^H \leftarrow \mathcal{D}^H \cup \{(s,o,h,s',h' )\}$}
            \STATE{$s \leftarrow s', \, h \leftarrow h'$}
        \ENDWHILE   
    \ENDFOR
\end{algorithmic}
\end{algorithm}

\section{Meta-Algorithm for High-and-Low-level Training} \label{sec:two-phase}
In this section, we provide a complete algorithm \emph{High-Level/Low-Level Meta-Learning} (HLML), able to learn both the high-level and the low-level policies in a provably efficient way.

HLML presented in Algorithm~\ref{alg:two-phase}, takes as input two optimal regret minimizers, \algname~and UCBVI \citep{azar2017minimax}, designed for learning in the FH-SMDP (i.e., at a high level, learning $\mu^*$) and in the FH-MDP (i.e., at a low level, learning $\bm{\pi}^*$), respectively. The meta-algorithm operates in $N$ \emph{stages}. In stage $n \in [N]$, we run the high-level regret minimizer for $K_n^H$ episodes, keeping the low-level policies $\bm{\pi}_{n-1} = (\pi_{n-1}^o)_{o \in \mathcal{O}}$ fixed (line \ref{alg-row:high}). \algname~will output the high-level policy $\mu_{n}$ which is chosen uniformly at random among the $\mu_{n,1},\dots,\mu_{n,K_n^H}$ played during its execution in the stage (line \ref{alg-row:hfixed}). Then, the control moves to the low level, and we run the low-level regret minimizer for $K_n^L$ episodes, keeping the high-level policy $\mu_n$ fixed (line \ref{alg-row:low}). UCBVI will output the low-level policies $\bm{\pi}_n$ chosen uniformly at random among the ones $\bm{\pi}_{n,1},\dots,\bm{\pi}_{n, K^L_n}$ played during its execution in the stage (line \ref{alg-row:lfixed}). The meta-algorithm, then, moves to the next stage $n+1$, passing back the control to the high level, and the process continues.

\begin{algorithm}[t]
\caption{High-Level/Low-level Meta-Learning (HLML)} \label{alg:two-phase}
\small
\begin{algorithmic}[1]
    \STATE{\textbf{Input:} $N$ phases, \algname~= $\mathfrak{A}^H$, UCBVI = $\mathfrak{A}^L$, and schedule $\forall n \in [N]: K_n^H = K_n^L = \lfloor 2^{n-1} \rfloor$}
    \STATE{Arbitrarily initialize $\mu_0$ and $\bm{\pi}_0$}
    \FOR{$n = 1, \dots, N$}
        \STATE Run $\mathfrak{A}^H$ on the FH-SMDP for $K_n^H$ episodes playing the sequence of high-level policies $\mu_{n,1},\dots,\mu_{n,K^H_n}$\label{alg-row:high}
        \vspace{-.35cm}
        \STATE Fix the high-level policy $\mu_n = \mu_{n,X}$ where $X \sim \mathrm{Uni}([K_n^H])$\label{alg-row:hfixed}
        \STATE Run $\mathfrak{A}^L$ on the FH-MDP for $K_n^L$ episodes playing the sequence of low-level policies $\bm{\pi}_{n,1},\dots,\bm{\pi}_{n,K^L_n}$\label{alg-row:low}
        \STATE{Fix the low-level policies $\bm{\pi}_{n-1} = \bm{\pi}_{n-1,Y}$ with $Y \sim \mathrm{Uni}([K^L_n])$}\label{alg-row:lfixed}
    \ENDFOR   
    \STATE{\textbf{return} $(\mu_N,\bm{\pi}_N)$}
\end{algorithmic}
\end{algorithm}

In order to achieve tight regret guarantees, we need to accurately select the schedule of the number of episodes $K_n^H$ and $K_n^L$, namely, we duplicate the number of episodes when moving from one stage $n$ to the next one $n+1$:
\begin{align}\label{eq:schedule}
\resizebox{.85\textwidth}{!}{$\displaystyle
    \forall n \in [N]:~~ K_n^H = K_n^L = \lfloor 2^{n-1} \rfloor \qquad \text{where}~ N = \lfloor \log_2 (2K+1) \rfloor~~ \text{and}~~ \sum_{n=1}^N K_n^H + K_n^L=K.
$}
\end{align}

The key feature of our meta-algorithm is that when the high-level algorithm is running in stage $n$ the low-level (inner-option) policies $\bm{\pi}_{n-1}$ are kept fixed. Therefore, \algname~is actually performing regret minimization in an FH-SMDP, enjoying the corresponding regret guarantees, for converging to the optimal high-level policy for the fixed options $\mathcal{O}$. This allows us to solve the common non-stationarity issues that arise when two learning processes are carried out in parallel. Clearly, such a high-level policy will not necessarily be $\mu^*$, since we are not guaranteed that the low-level policies $\bm{\pi}_{n-1}$ are optimal for the corresponding options. This is the reason why the execution of \algname~is stopped after $K_n^H$ episodes, and, within the same stage $n$, we proceed to run the low-level regret minimizer before continuing learning at the high-level. Similarly, in this phase, UCBVI is acting on the flat MDP with the goal of learning the inner policy $\pi^o_n$ for each of the options $o \in \mathcal{O}$. This amounts to solving for each option $o \in \mathcal{O}$ a single FH-MDP formalized as $\mathcal{M}_{o} = (\mathcal{S}_o, \mathcal{A}_o, p, r_o, H_o)$ where $\mathcal{S}_o \subseteq \mathcal{S}$, $\mathcal{A}_o \subseteq \mathcal{A}$, $H_o \leq H$, meaning that each option operates on a restricted portion of the original problem and for a specific fixed horizon $H_o$ (induced by $\mathcal{I}^o$ and $\beta^o$). This time the high-level policy is kept fixed, and consequently, its effect is enforcing a specific exploration that determines a particular option visitation. 

In principle, solving such FH-MDPs $\mathcal{M}_o$ can be as complex as solving the original problem $\mathcal{M}$ with a flat approach. This is expected since the advantages of a hierarchical approach emerge when a certain \emph{structure} on the original problem is present. This is particularly evident if we think of the convergence of the learning process of the low-level policies, which could potentially end up in a different optimum than the one reached by a flat approach in that same portion of the problem because the latter would have a complete scope over the whole problem. For this reason, a further assumption over the structure of the problem is required. 

\begin{assumption}\label{ass:continuity}
    For any optimal high-level policy $\mu^*$, let $\mathcal{O}_{\mu^*}$ the set of options played by $\mu^*$ and for $o \in \mathcal{O}_{\mu^*}$, let $\Pi_o^*$ the set of optimal low-level policies form the joint optimization. Let $\Pi_o^\#$ be the set of optimal low-level policies from the local optimization $(\pi_o^\#\in \argmax_{a \in \mathcal{A}} Q^{*,o}(s,a) \forall s \in \mathcal{S}_o)$. It is assumed that
    \begin{align} 
        \Pi_o^\# \subseteq  \Pi_o^*.
    \end{align}
\end{assumption}
This assumption ensures that the optimal inner-option policies $\pi^*_o$, on a portion of the original MDP $\mathcal{M}_o$ induced by an options $o \in \mathcal{O}$, selected by the optimal SMDP policy $\mu^*$, do not differ from an optimal policy $\pi^*$ of the {flat} problem. This way, we can safely learn in the FH-MDPs $\mathcal{M}_o$ knowing that the learned policy will be ``a portion'' of the optimal policy $\pi^*$ in the flat FH-MDP. This assumption, seemingly demanding, is the first one, to the best of our knowledge, that attempts to characterize a structural property of the FH-MDPs that is suitable for being addressed by means of a hierarchical approach. Indeed, if Assumption~\ref{ass:continuity} is violated, the inner-option learning deviates from the process of learning the optimal policy in the flat MDP, possibly preventing the convergence to the optimal policy in the hierarchical architecture. 
An example of a scenario in which this assumption is valid is the taxi problem described above. For instance, from a starting point A to destination B, the optimal driving policy (i.e., the one solving the subtask (i)) does not differ if the problem is considered a whole or a smaller one that includes just the neighborhood of the two points.


\textbf{Theoretical Analysis}~~As described above, in each stage $n \in [N]$, the learning process alternates between the high- and the low-level learning problems, keeping the other fixed. This induces a bias in both optimizations. To make this clear, we provide a convenient decomposition of the regret, which highlights the contributions of the two phases of learning in each stage:
\begin{align}\label{eq:eqeq}
    \text{Regret} (\text{HLML}, K)  = \sum_{n=1}^N  \Bigg(\sum_{k=1}^{K_n^H} \underbrace{ V^*_{\bm{*}}(s_1,1) -  V^{\mu_{n,k}}_{\bm{\pi}_{n-1}}(s_1,1)}_{\text{Regret during high-level learning}}+ \sum_{k=1}^{K_n^L}  \underbrace{V^*_{\bm{*}}(s_1,1) - V^{\mu_n}_{\bm{\pi}_{n,k}}(s_1,1)}_{\text{Regret during low-level learning}}\Bigg),
\end{align}
where $\mu_{n,k}$ and $\bm{\pi}_{n,k}$ are the high-level policy and the low-level policies played by the corresponding algorithms \algname~and UCBVI at episode $k$ of phase $n$. Unfortunately, the two terms in Equation~\eqref{eq:eqeq} cannot be directly bounded in terms of the properties of the regret minimization algorithms. This is because each of them, as explained above, will converge to the corresponding high/low-level optimal policy, given that the other-level policy is fixed. Thus, further elaboration is needed to highlight the bias terms:
\begin{align}\label{eq:eqeqeqA}
    &\underbrace{ V^*_{\bm{*}}(s_1,1) -  V^{\mu_{n,k}}_{\bm{\pi}_{n-1}}(s_1,1)}_{\text{Regret during high-level learning}} = \underbrace{ V^*_{\bm{*}}(s_1,1) -  V^{*}_{\bm{\pi}_{n-1}}(s_1,1)}_{\text{Bias of not playing $\bm{\pi}^*$}} +\underbrace{ V^*_{\bm{\pi}_{n-1}}(s_1,1) -  V^{\mu_{n,k}}_{\bm{\pi}_{n-1}}(s_1,1)}_{\text{Regret of \algname}} \\
    \label{eq:eqeqeqB}
    &\underbrace{ V^*_{\bm{*}}(s_1,1) - V^{\mu_n}_{\bm{\pi}_{n,k}}(s_1,1)}_{\text{Regret during low-level learning}} = \underbrace{ V^*_{\bm{*}}(s_1,1) - V^{\mu_n}_{\bm{*}}(s_1,1)}_{\text{Bias of not playing $\mu^*$}} + \underbrace{ V^{\mu_n}_{\bm{*}}(s_1,1) - V^{\mu_n}_{\bm{\pi}_{n,k}}(s_1,1)}_{\text{Regret of UCBVI}},
\end{align}
Thus, the regrets of the two phases (low- and high-level learning) are decomposed into a proper \emph{regret} term and a \emph{bias} term, which accounts for the fact that the other level is kept fixed. The regret terms can be easily managed by resorting to the properties of the regret minimizers. Concerning the bias terms, the high level corresponds to the value difference between playing the current low-level policies $\bm{\pi}_{n-1}$ compared to playing the optimal ones $\bm{\pi}^*$. Symmetrically, for the low level, this bias translates into the value difference between playing the current high-level policy $\mu_{n}$ compared to the optimal one $\mu^*$. From a technical perspective, we decide to upper bound the bias terms with the proper regret terms at the price of introducing a \emph{concentrability} coefficient for accounting of the distribution shift, as shown in the following result.
\begin{restatable}{lemma}{biasregret}\label{lemma:bias-regret}
    Let us define the concentrability coefficients:
    \begin{align}
        &\resizebox{.4\textwidth}{!}{$\displaystyle C^H \coloneqq \max_{n \in [N]} \inf_{\mu^* \text{ optimal}} \max_{(s,h) \in \mathcal{S} \times [H]} \frac{d^{\mu^*}_{s_1,1}(s,h)}{d^{\mu_{n}}_{s_1,1}(s,h)},$} \\
        &\resizebox{.55\textwidth}{!}{$\displaystyle C^L \coloneqq \max_{n \in [N]} \max_{o \in \mathcal{O}}  \inf_{\pi^*_o \text{ optimal}} \max_{(s,h) \in \mathcal{I}^o} \max_{(s',h') \in \mathcal{S}_o \times [H_o]}   \frac{d^{\pi^*_o}_{s,h}(s',h')}{d^{\pi_{n-1}^o}_{s,h}(s',h')}$}.
    \end{align}
    Then, it holds that:
    \begin{align}
        &\resizebox{.51\textwidth}{!}{$\displaystyle \underbrace{ V^*_{\bm{*}}(s_1,1) -  V^{*}_{\bm{\pi}_{n-1}}(s_1,1)}_{\text{Bias of not playing $\bm{\pi}^*$}} \le  C^H \Big( \underbrace{ V^{\mu_n}_{\bm{*}}(s_1,1) - V^{\mu_n}_{\bm{\pi}_{n-1}}(s_1,1)}_{\text{Regret of low-level algorithm}}\Big)$},\\
        &\resizebox{.51\textwidth}{!}{$\displaystyle  \underbrace{ V^*_{\bm{*}}(s_1,1) - V^{\mu_n}_{\bm{*}}(s_1,1)}_{\text{Bias of not playing $\mu^*$}} \le C^L \Big(\underbrace{ V^*_{\bm{\pi}_{n-1}}(s_1,1) -  V^{\mu_{n}}_{\bm{\pi}_{n-1}}(s_1,1)}_{\text{Regret of high-level algorithm}}\Big).$}
    \end{align}
\end{restatable}
Please note that the \emph{concentrability coefficients}, $C^H$ and $C^L$, are defined exclusively for state-stage pairs. They are ensured to be finite when all state-stage pairs are visited with non-zero probability under any policy. Additionally, they are proportional to $1/p_{\min}$, where $p_{\min}>0$ represents the minimum probability of visiting a state-stage pair with any policy.

We are finally ready to state the main theoretical guarantees on the regret of our meta-algorithm. 

\begin{restatable}{theorem}{twoucbvi}\label{thm:two-ucbvi}
Let $\mathcal{SM} = (\mathcal{S, O}, p^H, r^H, H)$ be an FH-SMDP and let $\mathcal{O}$ be a set of options to be learned inducing the FH-MDPs $\mathcal{M}_o = (\mathcal{S}_o, \mathcal{A}_o, p, r_o, H_o)$ for $o \in \mathcal{O}$. The regret suffered by Algorithm~\ref{alg:two-phase} under Assumption~\ref{ass:continuity}, episode schedule as in Equation~\eqref{eq:schedule}, and where $H_O = \max_{o \in \mathcal{O}} H_o$, is bounded with probability at least $1-\delta$ by:
\begin{align}\label{eq:two-ucbvi} 
    R(\text{HLML}, K) \leq \Tilde{O} &\bigg(C^L \underbrace{H\sqrt{SOKd}}_{\text{High-Level Regret}} + C^H \underbrace{H_O\sqrt{OSAKH_O}}_{\text{Low-Level Regret}}\bigg).
\end{align}
\end{restatable}

Some observations are in order. First, we relate the regret of the meta-algorithm in terms of the regret suffered by the individual regret minimizers, \algname~and UCBVI, weighted by the \emph{concentrability coefficients} $C^H$ and $C^L$. To be precise, the low-level regret is not the exact regret of UCBVI. It is the sum of the regret of the UCBVI instances run on all the options played in the $n^{th}$ phase, then summed for all the $N$ phases.
Second, we can now appreciate the role of Assumption~\ref{ass:continuity}. Indeed, in order to be able to converge at a low level to the optimal inner-option policies $\bm{\pi}^*$ (as in Equation~\eqref{eq:eqeqeq}), it must happen that the low-level regret minimizer performs an optimization that is compliant with what would have happened if solving the original flat MDP.

At this point, it is possible to properly characterize the class of problems more efficiently solvable with this HRL approach instead of a \emph{flat} one. We can do so by relating the regret of Theorem~\ref{thm:two-ucbvi}, with the lower bound in FH-MDPs \citep{osband2016lower} for non-stationary transitions.  
Let us consider a particular case for which $H_O = \alpha H$, with $0<\alpha < 1$, we can write:
\begin{align}\label{eq:comparison}
    \frac{\text{Regret of Theorem~\ref{thm:two-ucbvi}}}{\text{Lower Bound FH-MDPs}} \leq \frac{C^LH\sqrt{SOKd} + C^HH_O\sqrt{OSAKH_O}}{H\sqrt{SAKH}}
    = C^L\sqrt{\frac{Od}{AH}} + C^H\sqrt{O\alpha^3}
\end{align}

Therefore, considering Equation \eqref{eq:comparison}, the classes of problems for which this HRL approach will outperform the \emph{flat} one are the ones that guarantee to have this ratio smaller than 1 and with a structure compliant to Assumption \ref{ass:continuity}. Under the assumption that the effect of the concentrability coefficients is negligible, there is a clear advantage of using the hierarchical approach when the structure that the options induce on the MDP guarantees $Od \ll AH$ and $\sqrt{O\alpha^3}$ to be small enough. In other words, the advantage emerges when the number of options is significantly smaller than the number of primitive actions, and their durations significantly reduce the planning horizon in the SMDP problem.
Of course, given the presence of $C^L$ and $C^H$, this advantage gets mitigated by the magnitude of these constants. However, our conjecture is that with these coefficients, we can identify the point at which the convenience of HRL emerges, emphasizing the influence of the joint learning process besides the MDP's structure. This point would probably open a new question for the theoretical study of HRL.

\section{Conclusions}
In this paper, we investigated the problem of learning the inner-option policies together with learning the high-level policy in an HRL setting based on the options framework. We first provided Options-UCBVI, a novel, provably efficient algorithm for learning in finite-horizon SMDPs enjoying favorable regret guarantees, which become nearly tight when applied to standard FH-MDPs. Then, we combined Options-UCBVI and UCBVI into a novel meta-algorithm HLML based on the alternation between high- and low-level learning whose theoretical guarantees depend on those of the individual regret minimizers under particular structural assumptions of the problem. This assumption represents the first attempt to characterize the structure that an MDP should have to make a \emph{hierarchical} RL approach provably convenient compared to a \emph{flat} one. We succeeded in achieving sublinear regret for learning at both (high and low) levels, also showing the advantages over the resolution of the FH-MDP with a flat approach. One of the main limitations of the approach lies in the need for the concentrability coefficients in the analysis of the meta-algorithm.
Future works should investigate further in this direction to understand whether this represents an artifact of our analysis, a limitation of the algorithm, or an inherent challenge of the setting.

\subsubsection*{Acknowledgements}
Funded by the European Union – Next Generation EU within the project NRPP M4C2, Investment 1.3 DD. 341 -  15 March 2022 – FAIR – Future Artificial Intelligence Research – Spoke 4 - PE00000013 - D53C22002380006.


\bibliography{main}
\bibliographystyle{rlc}

\newpage
\appendix
\onecolumn

\section{Related Works}\label{app:relworks}
There is a vast literature for provably efficient algorithms for FH-MDP. \cite{osband2016lower} proves the lower bound for the regret in the FH-MDP setting, $\Omega(\sqrt{HSAT})$. Then, many works propose algorithms with guarantees that nearly close the problem, i.e., with upper bounds of the same order as the lower bound \citep{zanette2018}. \cite{azar2017minimax} definitively close the problem by proposing an innovative analysis of an algorithm for which the upper bound, $O(\sqrt{HSAT})$, matches the lower bound in all terms.

Nevertheless, only some works focused on theoretically understanding the benefits of hierarchical reinforcement learning approaches, and most of them consider a known set of pre-trained policies. In \cite{fruit2017exploration}, the authors propose an adaptation of UCRL2 \citep{auer2008near} for SMDPs. This work was the first to theoretically compare options instead of primitive actions to learn in SMDPs. It provides both an upper bound for the regret suffered by their algorithm and a lower bound for the general problem. However, it focuses on the average reward setting to study how to possibly induce a more efficient exploration when using a set of fixed options. Differently, we aim to analyze the advantages of using options to reduce the sample complexity of the problem, resorting to the intuition that temporally extended actions can intrinsically reduce the planning horizon in FH-SMDPs, and characterize problems likely to benefit from using HRL even when no prior information about the problem is known, up to its structure. \cite{fruit2017regret} is an extension of this work, where the need for prior knowledge of the distribution of cumulative reward and duration of each option is relaxed. However, the setting is identical.
Furthermore, \cite{mann2015approximate} studies the convergence property of Fitted Value Iteration (FVI) using temporally extended actions, showing that a longer options duration and pessimistic value function estimates lead to faster convergence. \cite{wen2020efficiency} demonstrate how patterns and substructures in the MDP provide benefits in terms of planning speed and statistical efficiency. They present a Bayesian approach that exploits this information, analyzing how sub-structure similarities and sub-problems' complexity contribute to the regret of their algorithm. A very recent approach proposed by \cite{robert2024sample} studies the sample complexity of a particular sub-class of HRL approaches: the Goal-conditioned one, in which a goal-based problem is structured into a hierarchy of sub-tasks, each with its own sub-goal. They analyzed the best possible performance achievable by the best algorithm in the worst possible problem by adapting to this framework the lower bound on the sample complexity presented by \cite{dann2015sample}. Nevertheless, this work is not completely related to our framework, which is more general than the goal-conditioned one.

The closest approach in the literature is \cite{drappo2023an}. They propose to relax the assumption of having a set of pre-trained options by implementing an Explore-Then-Commit approach \citep{lattimore2020bandit}, which first learns each options' policy and then exploits an adaptation of UCRL2 to FH-SMDPs \citep{auer2008near} to find the optimal policy over options. 
Nevertheless, they sacrifice optimality to relax this assumption. Indeed, their approach suffer from the standard sub-optimality of Explore-Then-Commit approaches, having a regret scaling with $K^{2/3}$, and additionally is suboptimal in $\sqrt{HS}$ being the high-level algorithm used in the second phase based on UCRL2.
Therefore, our approach is the first in the literature able to relax the aforementioned assumption maintaining optimal guarantees.

\section{Proof of the regret of \algname}\label{proof1}
In this section, we will present the analysis of the upper bound on the regret paid by \algname. The analysis will adapt the one of UCBVI \cite{azar2017minimax} to the FH-SMDP for non-stationary transition models. For simplicity, we will write $o = \mu_k(s,h)$, and $P^{\mu_k}(s',h'|s, h) = P(s',h'| s, \mu_k(s), h)$.

\optucbvi*
\begin{proof}
    The Proof follows the same ideas as the proofs of UCBVI for the Bernstein-Freedman exploration bonus.
    We can write the regret as:
    \begin{align}
        Regret(K) \leq \widetilde{Regret}(K) \leq \sum^K \Tilde{V}^{\mu_k}(s,1) - V^{\mu_k}(s,1)
    \end{align}
    Where $\Tilde{V}^{\mu_k}(s,1)$ is the optimistic value function, and $V^{\mu_k}(s,1)$, is the real value function considering the policy learned at the $k^{th}$ step.
    Following the analysis of the original paper we can write the regret in terms of the per step regret $\Tilde{\Delta}_{hk}(s_{hk})$. Thus,
    \begin{align}
        \widetilde{Regret}(K) \leq \sum_{i=1}^K\sum_{j=1}^H \Tilde{\Delta}_{ij}(s_{ij})
    \end{align}
    where the summation over $H$ is composed of $d$ terms, for the temporally extended transitions, where $d$ is a random variable describing the expected number of options played in one episode, refer to the main paper for a more detailed explanation (Section \ref{sec:opt-ucbvi}).\\
    Now let's define properly the per step regret:
    \begin{align*}
        \Tilde{\Delta}_{hk}(s_{ij}) &= \Tilde{V}^{\mu_k}(s_{hk}, h) - V^{\mu_k}(s_{hk}, h) \\
        &\myeq{a} [\hat{P}^{\mu_k}_{hk} \Tilde{V}^{\mu_k}(s',h')](s_{hk}) + b_{hk} - [P_h^{\mu_k} V^{\mu_k}(s',h')](s_{hk}) \textcolor{blue}{\ \pm \ [P^{\mu_k}\Tilde{V}^{\mu_k}(s',h')](s_{hk})}\\
        &= [(\hat{P}^{\mu_k}_{hk} - P_h^{\mu_k})\Tilde{V}^{\mu_k}(s',h')](s_{hk}) + b_{hk} + [P_h^{\mu_k}(\Tilde{V}^{\mu_k}(s',h') - V^{\mu_k}(s',h'))](s_{hk}) \\
        & \quad \textcolor{blue}{\ \pm \ [\Delta_p V^*(s',h')](s_{hk})}\\
        &= [(\hat{P}^{\mu_k}_{hk} - P_h^{\mu_k})(\Tilde{V}^{\mu_k}(s',h') - V^*(s',h')](s_{hk}) + b_{hk} +  P_h^{\mu_k} \Tilde{\Delta}_{h',k}(s_{hk}) \\
        & \quad + [(\hat{P}^{\mu_k}_{hk} - P_h^{\mu_k}) V^*(s',h')] (s_{hk}) \textcolor{blue}{\ \pm \ \Tilde{\Delta}_{h',k}(s')}\\
        &\myeq{b} c_{hk} + b_{hk} + e_{hk} + \epsilon_{hk} + \Tilde{\Delta}_{h',k}(s')
    \end{align*}
    \begin{enumerate}[label=(\alph*)]
        \item By applying the bellman operator considering known reward that simplifies, and where $P^{\mu_k}_{h} = p(\cdot, \cdot |s_{h}, \mu_k(s_{h}), h)$, and $\hat{P}^{\mu_k}_{hk} = \hat{p}(\cdot, \cdot |s_{hk}, \mu_k(s_{hk}), h)$, the estimated transition model at episode $k$. By applying the bellman operator on the optimistic value function, the bonus term $b_{hk}$ is added to the reward.
        \item By defining $c_{hk} = [(\hat{P}^{\mu_k}_{hk} - P_h^{\mu_k})(\Tilde{V}^{\mu_k}(s',h') - V^*(s',h')](s_{hk})$, the correction term,  $e_{hk} = [(\hat{P}^{\mu_k}_{hk} - P_h^{\mu_k}) V^*(s',h')](s_{hk})$ the estimation error of the optimal value function, and $\epsilon_{hk}$ a martingale difference, defined as $\epsilon_{hk} = \mathcal{M}_t \Tilde{\Delta}_{h',k}(s) = P^{\mu_k}_h \Tilde{\Delta}_{h',k}(s) - \Tilde{\Delta}_{h',k}(s')$, where $\mathcal{M}_t$ is defined as a martingale operator (refer to appendix B.3 of \cite{azar2017minimax}).
    \end{enumerate}

    Let us now bound each of these terms separately.
    \subsection{Bound of the correction term $c_{hk}$}
    In this subsection, we bound the correction term
    \begin{align*}
        c_{hk} &= [(\hat{P}^{\mu_k}_{hk} - P_h^{\mu_k})(\Tilde{V}^{\mu_k}(s',h') - V^*(s',h')](s_{hk}) \\
        &\myeq{a} \sum_{s' \in S} \sum_{h' \in H} (\hat{P}_k^{\mu_k}(s',h'|s_{hk}, h) - P^{\mu_k}(s',h'|s_{hk}, h))(\Tilde{V}^{\mu_k}(s', h') - V^*(s', h')) \\
        &\myleq{b} \sum_{s' \in S} \sum_{h' \in H} \left(2\sqrt{\frac{p_{hk}(s')(1-p_{hk}(s')) L}{n_k(s,o,h)}} + \frac{4L}{3n_k(s,o,h)} \right) \Tilde{\Delta}_{h'k}(s') \\
        &\myleq{c} 2 \sqrt{L}\sum_{s' \in S} \sum_{h' \in H} \sqrt{\frac{p_{hk}(s')}{n_k(s,o,h)}} \Tilde{\Delta}_{h'k}(s') + \frac{4SH^2L}{3n_k(s,o,h)} \\
        &\myeq{d} 2 \sqrt{L}\bigg( \sum_{(s',h') \in [(s',h')]_{typ}} \sqrt{\frac{p_{hk}(s')}{n_k(s,o,h)}}\Tilde{\Delta}_{h'k}(s') \\
        & \quad + \sum_{(s',h') \notin [(s',h')]_{typ}} \sqrt{\frac{p_{hk}(s')}{n_k(s,o,h)}}\Tilde{\Delta}_{h'k}(s') \bigg) + \frac{4SH^2L}{3n_k(s,o,h)} \\
    \end{align*}
    \begin{align*}
        &\myeq{e} 2\sqrt{L} \bigg( \sum_{(s',h') \in [(s',h')]_{typ}} P^{\mu_k}(s',h'|s_{hk}, h') \sqrt{\frac{1}{p_{hk}(s')n_k(s,o,h)}}\Tilde{\Delta}_{h'k}(s') \\
        & \quad + \sum_{(s',h') \notin [(s',h')]_{typ}} \sqrt{\frac{p_{hk}(s')n_k(s,o,h)}{n_k(s,o,h)^2}}\Tilde{\Delta}_{h'k}(s') \bigg) + \frac{4SH^2L}{3n_k(s,o,h)}\\
        &\myeq{f} 2\sqrt{L}\bigg( \Bar{\epsilon}_{hk} + \sqrt{\frac{1}{p_{hk}(s')n_k(s,o,h)}}\mathbb{I}((s',h') \in [(s'h')]_{typ})\Tilde{\Delta}_{h'k}(s') \\
        & \quad +  \sum_{(s',h') \notin [(s',h')]_{typ}} \sqrt{\frac{p_{hk}(s')n_k(s,o,h)}{n_k(s,o,h)^2}}\Tilde{\Delta}_{h'k}(s') \bigg) + \frac{4SH^2L}{3n_k(s,o,h)}\\
        &\myleq{g} 2\sqrt{L}\left(\Bar{\epsilon}_{hk} + \sqrt{\frac{1}{4LH^2}}\Tilde{\Delta}_{h'k}(s') + \frac{SH^2\sqrt{4LH^2}}{n_k(s,o,h)} \right) + \frac{4SH^2L}{3n_k(s,o,h)} \\
        &\leq 2\sqrt{L}\Bar{\epsilon}_{hk} + \frac{1}{H}\Tilde{\Delta}_{h'k}(s') + \frac{4SH^3L}{n_k(s,o,h)} + \frac{4SH^2L}{3n_k(s,o,h)}
    \end{align*}
    
    \begin{enumerate}[label=(\alph*)]
        \item By considering, for brevity, $P^{\mu}(s',h'|s, h) = P(s',h'| s, \mu(s), h)$, and summing over all the possible next states and next stages.
        \item Where for the first term we substitute the difference of transition probabilities with the relative confidence interval (refer to section B.4 on the appendix of \cite{azar2017minimax}), $\big|\hat{P}_k^{\mu_k}(s',h'|s_{hk}, h) - P^{\mu_k}(s',h'|s_{hk}, h)\big| \leq 2\sqrt{\frac{p_{hk}(s')(1-p_{hk}(s')) L}{n_k(s,o,h)}} + \frac{4L}{3n_k(s,o,h)}$, where $p_{hk}(s') = P^{\mu_k}(s',h'|s, h)$. Then we can bound $\Tilde{V}^{\mu_k}(s', h') - V^*(s', h')$ with $\Tilde{\Delta}_{h'k}(s')$ because $V^*(s', h') \geq V^{\mu_k}(s',h')$ (the true value function of the policy $\mu_k$) by definition.
        \item Because $(1-p_{hk}(s')) \leq 1$ and $\Tilde{\Delta}_{h'k}(s') \leq H$
        \item We divide the summation over all the possible next state-stage, in the summation over the pairs contained in the typical pairs and the ones outside the set (the typical episodes are the episodes in which we have smaller regret; refer to the appendix of \cite{azar2017minimax}).
        \item We multiply the first term by $\frac{p_{hk}(s')}{p_{hk}(s')}$, and the second by $\frac{n_k(s,o,h)}{n_k(s,o,h)}$.
        \item We sum and subtract $\sqrt{\frac{\mathbb{I}((s',h') \in [(s'h')]_{typ})}{p_{hk}(s')n_k(s,o,h)}}\Tilde{\Delta}_{h'k}(s')$ and apply the martingale operator $\mathcal{M}$ (see (b) in the previous proof). $\Bar{\epsilon}_{hk} = P^{\mu_k}_h \sqrt{\frac{\mathbb{I}((s',h') \in [(s'h')]_{typ})}{p_{hk}(s')n_k(s,o,h)}}\Tilde{\Delta}_{h'k}(s')+ \sqrt{\frac{\mathbb{I}((s',h') \in [(s'h')]_{typ})}{p_{hk}(s')n_k(s,o,h)}}\Tilde{\Delta}_{h'k}(s')$.
        \item For typical next state-stage pairs $n_k(s,o,h)P(s',h'|s,o,h) \geq 2H^2L$, where $L$ is a logarithmic term (We kept the same lower bound of  \cite{azar2017minimax}).
    \end{enumerate}

    Now, before bounding the estimation error and the exploration bonus, let's rewrite the regret as 
    \begin{align*}
        \widetilde{Regret}(K) &= \sum_{i=1}^K\Tilde{\Delta}_{1 i}(s_1) = \sum_{i=1}^K\sum_{j=1}^H \Tilde{\Delta}_{ij}(s_{ij}) \\
        &\leq \underbrace{\bigg(1+ \frac{1}{H}\bigg)^d}_{\leq e} \sum_{i=1}^K\sum_{j=1}^H \left(b_{hk} + e_{hk} + \epsilon_{hk} + 2\sqrt{L} \Bar{\epsilon}_{hk} + \frac{4SH^3L}{n_k(s,o,h)} + \frac{4SH^2L}{3n_k(s,o,h)}\right)
    \end{align*}
    or otherwise omitting the last term which is dominated
    \begin{align}\label{eq:tot-regret}
        \widetilde{Regret}(K) \leq \sum_{i=1}^K\sum_{j=1}^H \left(b_{hk} + e_{hk} + \epsilon_{hk} + 2\sqrt{L} \Bar{\epsilon}_{hk} + \frac{4SH^3L}{n_k(s,o,h)} \right)
    \end{align}

    \subsection{Bound of the estimation error $e_{hk}$}
    Let's consider just the typical episodes, the episodes for which the number of visits of state-option-stage pairs is larger than the rest of the episodes. 
    \begin{align*}
        \sum_{k=1}^K\sum_{h=1}^H e_{hk} &= \sum_{k=1}^K \sum_{h=1}^H \mathbb{I}(k \in [k]_{typ})([(\hat{P}^{\mu_k}_{hk} - P_h^{\mu_k}) V^*(s',h')](s_{hk})) \\
        &\myleq{a} \sum_{k=1}^K \sum_{h=1}^H \mathbb{I}(k \in [k]_{typ}) \bigg( 2 \sqrt{\frac{\mathbb{V}_{hk}^* L}{n_k(s_{hk},o,h)}} + \frac{4HL}{3n_k(s,o,h)}\bigg) \\
        & \myleq{b} 2\sqrt{L} \sqrt{\sum_{k=1}^K \sum_{h=1}^H \mathbb{V}_{hk}^*} \sqrt{\sum_{k=1}^K \sum_{h=1}^H \mathbb{I}(k \in [k]_{typ}) \frac{1}{n_k(s,o,h)}} \\
        & \quad + \sum_{k=1}^K \sum_{h=1}^H \mathbb{I}(k \in [k]_{typ}) \frac{4HL}{3n_k(s,o,h)} \\
        &\myleq{c} 2 \sqrt{L} \Big(\sqrt{KH^2 + HdU_{K,1} + \square \sqrt{H^5KL} + 4/3 H^3L}\Big) \Big(\sqrt{2SOdL}\Big) + 4/3 HSOdL^2 \\
        &\myleq{d} \square LH \sqrt{KSOd} + \square Ld \sqrt{HSOU_{K,1}}
    \end{align*}
    \begin{enumerate}[label=(\alph*)]
        \item Using Bernstein Inequality. $\V^*_{hk} = \Var_{(s',h') \sim P^{\mu_k}(\cdot|s, h)}(V^*(s',h'))$ (Remember the meaning of $P^{\mu_k}$)
        \item Using Cauchy-Schwartz inequality
        \item Summing and subtracting $\V^{\mu_k}_{hk} = \Var_{(s',h') \sim P^{\mu_k}(\cdot|s, h)}(V^{\mu_k}(s',h'))$ the variance of the next state-stage pair value function, inside the first square root, and then using Lemma \ref{lemma:ucbvi-8} and \ref{lemma:ucbvi-9}. For the second square root and the additional term, we just use a pigeon-hole argument (Lemma \ref{lemma:pigeon}). We ignore the numerical constant represented as $\square$.
        \item Because for typical episodes $K \geq H^2L^2S^2Od$ and thus we consider only the dominant terms.
    \end{enumerate}
    
    \subsection{Bound of the martingale differences $\epsilon_{hk}$ and $\Bar{\epsilon}_{hk}$}
    \begin{align}
        \sum_{k=1}^K\sum_{h=1}^H \epsilon_{hk} &\leq H\sqrt{dKL} \\
        \sum_{k=1}^K\sum_{h=1}^H \Bar{\epsilon}_{hk} &\leq \sqrt{dK}
    \end{align}
    These results follow the same proofs of the original paper, thus considering the same event $\mathcal{E}$ to hold. The only difference is that the summation over $H$ is a summation of $d$ elements, and thus, $(H-h)$ is at most $d$ in this case for the effect of the temporally extended actions.

    \subsection{Second-order term}
    Let's now see the upper bound on the second-order term, which will be useful for the upper bound on the exploration bonus. \\
    By applying the pigeon-hole principle (Lemma \ref{lemma:pigeon}).
    \begin{align}
        \sum_{k=1}^K\sum_{h=1}^H \frac{4SH^3L}{n_k(s,o,h)} \leq \square H^3S^2OL^2d
    \end{align}

    \subsection{Bound of the exploration bonus $b_{hk}$}
    Before bounding the sum, we need to define the exploration bonus. 
    We will consider an adaptation to temporally extended actions and non-stationary transitions of the same bonus presented in the original paper of UCBVI \cite{azar2017minimax}. However, to make the definition clearer, let us motivate the need for this term. \\
    Given that the optimistic value function $\tilde{V}^{\mu_k}$ is an upper bound of the true value function $V^*$, we can not guarantee the same for the relative empirical variance. Hence, if the empirical variance of $\tilde{V}^{\mu_k}$ is an upper bound on the empirical variance of $V^*$. Nonetheless, it is possible to prove that when the two value functions are sufficiently close to each other, the same applies to their empirical variance. \\
    Let's resort to Lemma 2 of \cite{azar2017minimax}, 
    \begin{align*}
        \hat{\mathbb{V}}^*_{hk} \leq 2 \hat{\mathbb{V}}_{hk} + 2 \Var_{(s',h') \sim \hat{P}^{\mu_k}} (\Tilde{V}(s',h') - V^*(s',h')) \leq 2 \hat{\mathbb{V}}_{hk} + 2 \hat{P}^{\mu_k}(\Tilde{V}(s',h') - V^*(s',h'))^2
    \end{align*} 
    where $\hat{\mathbb{V}}^*_{hk} = \Var_{(s',h') \sim P^{\mu_k}(\cdot|s, h)}(V^*(s',h'))$ and   $\hat{\mathbb{V}}_{hk} = \Var_{(s',h') \sim \hat{P}_k^{\mu_k}}(\tilde{V}^{\mu_k}(s,h))$.\\
    We need this term to be of the same order as the estimation error $e_{hk}$, and thus we can say that 
    \begin{align}
        b_{hk} \sim [(\hat{P}^{\mu_k}_{hk} - P_h^{\mu_k}) V^*(s',h')](s_{hk})
    \end{align}
    This time, however, we use the Empirical-Bernstein inequality \cite{maurer2009empirical} because we need the empirical variance to appear.
    \begin{align}
        b_{hk} \leq \bigg( 2 \sqrt{\frac{\hat{\mathbb{V}}_{hk}^* L}{n_k(s,o,h)}} + \frac{14HL}{3n_k(s,o,h)}\bigg)
    \end{align}
    By applying Lemma 2 to this equation and substituting $\hat{\mathbb{V}}_{hk}^*$ we get the same form of bonus of \cite{azar2017minimax}. 
    \begin{align*}
        \resizebox{.95\textwidth}{!}{$\displaystyle
        b_{hk} = \sqrt{\frac{8L\Var_{(s',h') \sim \hat{P}_k^{\mu_k}(\cdot|s,h)}(\tilde{V}^{\mu_k}(s',h')}{n_k(s,o,h)}} + \frac{14HL}{3n_k(s,o,h)} + \sqrt{\frac{8 \sum_{s',h'} \hat{P}_k^{\mu_k}(s',h'|s,h)\big[ \min \left( b'_{h'k},H^2\right) \big]}{n_k(s,o,h)}}$}
    \end{align*}
    in which $b'_{hk}$ stands for the upper bound on the square root of the difference between the optimistic value function in the next state-stage pair, and the optimal value function in the same next state-stage.

    The last thing to do to properly define the bonus is express $b'_{hk}$ in our scenario. Let's write
    \begin{align}
        \Tilde{V}(s',h') - V^*(s',h') \leq \sqrt{b'_{hk}}
    \end{align}
    and consider that $b'_{hk}$ has to be appropriate to guarantee an adaptation of Lemma 16 of \cite{azar2017minimax}, in which the second inequality applies if $\sqrt{N'_{hk}(s)} \geq 2500H^2S^2AL^2$, which is the second order term for standard UCBVI, given that $N'_{hk}(s) \geq H^2S^2AL^2$ for good episodes. Therefore, in our scenario, we need that
    \begin{align}
         \sqrt{b'_{hk}} \left(\sum_o n_k(s,o,h)\right) \geq \square H^4S^2OL^2 \geq \square H^3S^2OL^2d
    \end{align}
    where the r.h.s of the equation above is the second-order term in our case.
    Thus, considering that $\sum_o n_k(s,o,h) \leq K$, and $K \geq H^3L^2S^2O \geq H^2L^2S^2Od \ $ for typical episodes, we have:
    \begin{align}
        b'_{hk} = \frac{100^2 H^5S^2L^2O}{\sum_{o}n_k(s,o,h)}
    \end{align}
    When considering the bound for the next state-stage pair $b'_{h'k}$, we simply  refer to the visit count of the next state and next stage $n_k(s',o,h')$. The numerical constant $100^2$ is derived analogously to \cite{azar2017minimax}.

    Let's now analyze the summation of this term, considering, as for $e_{hk}$, just the typical episodes.
    \begin{align*}
        \sum_{k=1}^K\sum_{h=1}^H b_{hk} &= \underbrace{\sum_{k=1}^K \sum_{h=1}^H \mathbb{I}(k \in [k]_{typ}) \Bigg(\sqrt{\frac{8L\Var_{(s',h') \sim \hat{P}_k^{\mu_k}(\cdot|s,h)}(\tilde{V}^{\mu_k}(s',h'))}{n_k(s,o,h)}} + \frac{14HL}{3n_k(s,o,h)}\Bigg)}_{(ft)} \\
        & \quad + \underbrace{\sum_{k=1}^K \sum_{h=1}^H \mathbb{I}(k \in [k]_{typ}) \sqrt{\frac{8 \sum_{s',h'} \hat{P}_k^{\mu_k}(s',h'|s,h)\big[ \min \left( b'_{h'k},H^2\right) \big]}{n_k(s,o,h)}}}_{(st)}\\
    \end{align*}
    We separately analyze the first two terms and then the last.\\
    The analysis of $(ft)$ follows the same concept as the analysis conducted for the estimation error $e_{hk}$ where instead of using Lemma \ref{lemma:ucbvi-9} we use Lemma \ref{lemma:ucbvi-10}
    \begin{align*}
        (ft) &\myleq{a} \sqrt{8L} \left(\sqrt{KH^2 + \square HdU_{K,1} + \square H^2Sd\sqrt{KLO} + 4/3H^3L}\right) (\sqrt{SOdL}) + 14/3HSOdL^2\\
        & \myleq{b} \sqrt{8L} \left(\sqrt{KH^2 + \square HdU_{K,1}}\right) (\sqrt{SOdL}) + 14/3HSOdL^2\\
        & \leq \square LH\sqrt{KSOd} + \square Ld\sqrt{HSOU_{K,1}}
    \end{align*}
    \begin{enumerate}[label=(\alph*)]
        \item As we said above, we follow the same concept of point (c) of the proof of the upper bound of $e_{hk}$. In this case, we use Lemma \ref{lemma:ucbvi-10} instead of Lemma \ref{lemma:ucbvi-9}.
        \item Because for typical episodes $K \geq H^2L^2S^2Od$ and thus we consider only the dominant terms.
    \end{enumerate}
    Regarding the second term $(st)$ adapting the proofs of \cite{azar2017minimax}, we will focus only on the last term $(k)(h)$, which results in a term of the same order of the second-order term already analyzed, the other two terms are upper bounded by the main terms.
    \begin{align*}
        (st) &\myleq{a} \sqrt{\sum_{k=1}^K \sum_{h=1}^H \mathbb{I}(k \in [k]_{typ}) b'_{h'k}} \sqrt{\sum_{k=1}^K \sum_{h=1}^H \mathbb{I}(k \in [k]_{typ}) \frac{1}{n_k(s,o,h)}}\\
        &\myleq{b}\sqrt{H^5S^2L^2O}\sqrt{\sum_{k=1}^K \sum_{h=1}^H \mathbb{I}(k \in [k]_{typ}) \frac{1}{n_k(s',o,h')}} \sqrt{\sum_{k=1}^K \sum_{h=1}^H \mathbb{I}(k \in [k]_{typ}) \frac{1}{n_k(s,o,h)}} \\
        &\myleq{c} \sqrt{H^5S^2L^2O} (\sqrt{SOdL})^2 \\
        &= H^2S^2L^2\sqrt{O^3Hd^2} \\
        &\myleq{d} H^3S^2L^2Od
    \end{align*}
    \begin{enumerate}[label=(\alph*)]
        \item Considering only the $(k)(h)$ of the original proof and applying Cauchy-Schwartz inequality.
        \item By substituting $b'_{hk}$ in the equation.
        \item By applying two times Lemma \ref{lemma:pigeon}.
        \item If $O \leq H$.
    \end{enumerate}
    To conclude the summation of exploration bonuses
    \begin{align}
        \sum_{k=1}^K\sum_{h=1}^H b_{hk} \leq \square LH\sqrt{KSOd} + \square Ld\sqrt{HSOU_{K,1}} + H^3S^2L^2Od
    \end{align}
    neglecting smaller order terms.

    \subsection{Summing all the terms}
    Finally, we can combine all the terms analyzed separately back into Equation \eqref{eq:tot-regret}, and we will get:
    \begin{align*}
        \widetilde{Regret}(K) &\leq \square LH\sqrt{KSOd} + \square Ld\sqrt{HSOU_{K,1}} + \square H^3S^2L^2Od + H\sqrt{dKL} \\
        & \myleq{a}  \square LH\sqrt{KSOd} + \square HSL^2Od^2 + \square H^3S^2L^2Od + H\sqrt{dKL} \\
        & \leq  \square LH\sqrt{KSOd} + \square H^3S^2L^2Od + H\sqrt{dKL}
    \end{align*}
    where $(a)$ results by solving for $U_{K,1}$, and this completes the proof, ignoring the numeric constants replaced by $\square$.
\end{proof}

\textbf{Remark:}~~The term $d$ is a random variable, being the duration of each option a random variable itself. However, as shown in \cite{drappo2023an}, it is possible to bound this value when we have options with duration $\tau_{\min} \leq \tau_o \leq \tau_{\max}$, resorting to \emph{renewal processes} theory \citep{renewal2019} with
\begin{align*}
    d \le \sqrt{\frac{32H (\tau_{\max}- \tau_{\min})\log(2/\delta)}{\min_{o\in\mathcal{O}}\E[\tau_o]^3}} + \frac{H}{\min_{o \in \mathcal{O}}\E[\tau_o]}.
\end{align*}
holding with probability at least $1-\delta$.\\
This term is bounded by the ratio between the horizon $H$ and the expected duration of the shorter option composing the set, plus a confidence interval accounting for the stochasticity of the duration.

\section{Proof of Theorem \ref{thm:two-ucbvi}}\label{proof2}
In this section, we will provide a detailed proof of Theorem \ref{thm:two-ucbvi}.

As described in the main paper, the meta-algorithm alternates between two regret minimizers, UCBVI and \algname, for $N$ stages at two levels of temporal abstraction of the problem. While learning on one level, the policies of the second are kept fixed for all episodes on the stage.

Initially, we will keep the analysis general for any pair of regret minimizers, $\mathfrak{A}^L, \mathfrak{A}^H$ - where the former is the regret minimizer used for the low-level and the latter the one used for the high-level.

Before proceeding, we introduce Lemma \ref{lemma:bias-regret}, which relates the regret paid by the regret minimizer of one level to the bias introduced in the learning of the other level.

\biasregret*
where $\mu^*$ is the optimal high-level policy (SMDP), and $\pi_o^*$ is the optimal policy of a single option $o$ (low-level optimal policy).
\begin{proof}
    Let us write the bias of a level for the stage $n \in [N]$ as $\beta_n$, respectively specialized as $\beta_n^H$ for the high-level bias and $\beta_n^L$ for the low-level bias.
    \begin{align*}
        \beta_n^H &= V^*_{\bm{*}}(s_1,1) -  V^{*}_{\bm{\pi}_{n-1}}(s_1,1) \\
        &\myeq{a} \E_{(s,h) \sim d^{\mu^*}_{s_1,1}} \left[ R_{\pi^*}(s,h) - R_{\pi_{n-1}}(s,h) \right] \\
        &\myeq{b} \E_{(s,h) \sim d^{\mu_{n}}_{s_1,1}} \left[ \frac{d^{\mu^*}_{s_1,1}(s,h)}{d^{\mu_{n}}_{s_1,1}(s,h)} \big(R_{\pi^*}(s,h) - R_{\pi_{n-1}}(s,h)\big) \right]\\
        &\myleq{c} \max_{n \in [N]} \inf_{\mu^* \text{ optimal}} \max_{(s,h) \in \mathcal{S} \times [H]} \frac{d^{\mu^*}_{s_1,1}(s,h)}{d^{\mu_{n}}_{s_1,1}(s,h)} \left( V^{\mu_n}_{\bm{*}}(s_1,1) - V^{\mu_n}_{\bm{\pi}_{n-1}}(s_1,1)) \right)\\
        &\myleq{d} C^H \left( V^{\mu_n}_{\bm{*}}(s_1,1) - V^{\mu_n}_{\bm{\pi}_{n-1}}(s_1,1) \right)
    \end{align*}
    \begin{enumerate}[label=(\alph*)]
        \item We can write the difference in value as the difference in return of the two option policies, where $R_{\pi^*}$ and $R_{\pi_{n-1}}$ are respectively the return obtained by playing the optimal options policies, and the return obtained by playing the options policies learned up to the previous step, and the state-stage pairs $(s,h)$ are sampled from the distribution of visit induced by the policy $\mu^*$.
        \item Using an \emph{importance-sampling} argument, we can change the exploration policy by adding the \emph{importance weighting} term $\frac{d^{\mu^*}_{s_1,1}(s,h)}{d^{\mu_{n}}_{s_1,1}(s,h)}$
        \item Substituting the expectation with the \emph{sup} over the states and stages, the \emph{inf} over the possible optimal exploration policies, and maximizing for all possible n stages.
        \item Substituting the first term with the constant $C^H$, defined above.
    \end{enumerate}
    We will not consider the proof of the second inequality because it follows the same passages.
\end{proof}

Given this Lemma, we can provide a general result for any choice of $\mathfrak{A}^L, \mathfrak{A}^H$, and any choice of scheduling.

\begin{lemma}\label{lemma:meta-regretA}
    Let $\mathfrak{A}^H$ and $\mathfrak{A}^L$ be two regret minimizers that suffer regret bounded $R^H(K)$ and $R^L(K)$ when run for $K$ episodes. Then, under Assumption~\ref{ass:continuity}, Algorithm~\ref{alg:two-phase} when run with the episode schedule $(K_n^H,K_n^L)_{n=1}^N$ such that $\sum_{n=1}^N K_n^L+K_n^H = K$, suffers regret bounded by:
    \begin{align}
        \notag R(\text{HLML},K) \le \sum_{n=1}^N &\Big((C^H +1) R^L(K_{n}^L) + (C^L+1)R^H(K_n^H)\Big).
    \end{align}
\end{lemma}

\begin{proof}
    We can write the regret of the two-phase algorithm as a summation of the regret of the high-level and the regret of the low-level as expressed by Equation \eqref{eq:eqeq} in the main paper.
    \begin{align*}
        \text{Regret} (\text{HLML}, K)  &= \sum_{n=1}^N  \Bigg(\sum_{k=1}^{K_n^H} \left( V^*_{\bm{*}}(s_1,1) -  V^{\mu_{n,k}}_{\bm{\pi}_{n-1}}(s_1,1) \right) +  \sum_{k=1}^{K_n^L}  \left( V^*_{\bm{*}}(s_1,1) - V^{\mu_n}_{\bm{\pi}_{n,k}}(s_1,1)\right) \Bigg) \\
        &\myeq{a} \sum_{n=1}^N \left( \beta_n^H + R^H(K_{n}^H) + \beta_n^L + R^L(K_{n}^L) \right) \\
        &\myleq{b} \sum_{n=1}^N \left( C^H R^L(K_{n-1}^L) + R^H(K_n^H) + C^LR^H(K_{n-1}^H) + R^L(K_{n}^L) \right) \\
        &\myleq{c} \sum_{n=1}^N (C^H +1) R^L(K_{n}^L) + (C^L+1)R^H(K_n^H).
    \end{align*}
    \begin{enumerate}[label=(\alph*)]
        \item We can decompose the two terms of the summation as shown in Equations \eqref{eq:eqeqeqA} and \eqref{eq:eqeqeqB}, and then for shortness, use $\beta_n$ to express the bias of the two levels at the $n^{th}$ stage, and $R(K_n)$ for the regret of the two regret minimizers, $\mathfrak{A}^L, \mathfrak{A}^H$, at the $n^{th}$ stage.
        \item By applying Lemma \ref{lemma:bias-regret} for the two general regret minimizers.
        \item Clearly the sum of $n-1$ is smaller than the sum of $n$ terms, thus we can upper bound $R^L(K^L_{n-1})$ with $R^L(K^L_{n})$, and the same for $R^H(K^H_{n-1})$.
    \end{enumerate}
    And with the last step, we conclude the proof.
\end{proof}

Now we can specialize Lemma \ref{lemma:meta-regretA} for UCBVI for the options learning and \algname~for the high-level, and we get:

\twoucbvi*
\begin{proof}
    For the option learning procedure, we instantiate a UCBVI algorithm for each sub-MDP $\mathcal{M}_o$, and for the $n-th$ phase, we paid a regret proportional to:
    \begin{align}
        \sum_{k=1}^{K_n^L} R_{{o_k}_k} &= \sum_{o} \sum_{j=1}^{K_o}R_{oj}\\
        &\myeq{a} \sum_o H_o\sqrt{S_oA_oK_oH_o}\\
        &\myleq{b} H_O\sqrt{SAH_O}\sum_o\sqrt{K_o}\\
        &\myleq{c} H_O\sqrt{SAH_O}\sqrt{O\sum_oK_o}\\
        &= H_O\sqrt{OSAH_OK_n^L}
    \end{align}
    where $R_{{o_k}_k}$ is the regret paid for running the option $o_k$ in the $k-th$ episode and $K_o$ are the episodes given to that option $o$. With (a), we just write the regret of running UCBVI on $K_o$ episodes. In the passage (b), we upper bound to the worst possible sub-MDP, $\mathcal{M}_o$, where for the state space and the action space, we have the cardinalities of the primitive MDP, and we have an episode duration $H_O = \max_o H_o$. In the next inequality (c), we use the Cauchy-Schwartz inequality, and being $\sum_o K_o = K_n^L$ the last equality holds.
    Therefore, by considering just the dominant term of the two upper bounds of regret, we can write
    \begin{align*}
        R^L_{K^L_n} &= \textit{Regret-UCBVI} \leq \Tilde{O}\left(H_O\sqrt{OSAK^L_nH_O}\right) \\
        R^H_{K^H_n} &= \textit{Regret-O-UCBVI} \leq \Tilde{O}\left(H\sqrt{SOK^H_nd}\right) 
    \end{align*}
    Now by directly substituting these results in Lemma \ref{lemma:meta-regretA} and considering the scheduling proposed in Equation \eqref{eq:schedule}, we can rewrite the regret of the meta-algorithm as:
    \begin{align*}
         Regret(\text{HLML}, K) &\leq \Tilde{O} \left( \sum_{n=1}^N \left((C^H + 1)H_O \sqrt{OSAH_O 2^n} + (C^L +1)H\sqrt{SOd2^n} \right) \right) \\
         &= \Tilde{O} \left( \left((C^H + 1)H_O \sqrt{OSAH_O} + (C^L +1)H\sqrt{SOd} \right) \sum_{n=1}^N \sqrt{2^n} \right)\\
         &= \Tilde{O} \left(\left((C^H + 1)H_O \sqrt{OSAH_O} + (C^L +1)H\sqrt{SOd} \right) 2\sqrt{2}\sum_{n=0}^{N/2} 2^n \right)\\
         &= \Tilde{O} \left( \left((C^H + 1)H_O \sqrt{OSAH_O} + (C^L +1)H\sqrt{SOd} \right)  \left( 2\sqrt{2}(2^{N/2 +1} - 1) \right) \right)\\
         &\mypropto{a} \Tilde{O} \left( \left(C^HH_O \sqrt{OSAH_O} + C^L H\sqrt{SOd} \right) 2^{(\log_2 (K))/2} \right)\\
         &\leq \Tilde{O} \left(\left(C^HH_O \sqrt{OSAH_O} + C^L H\sqrt{SOd} \right) \sqrt{K} \right)
    \end{align*}
    Where all the passages follow algebraic operations, except for $(a)$ in which we neglect all the numerical constants and we consider that $K = 2\sum_{n=1}^N 2^{n-1} = 2^{N+1} - 1$ and thus, $N = \log_2(K)$. The last passage concludes the proof.
\end{proof}

\section{Useful Lemmas}
\begin{lemma}\label{lemma:pigeon}
    Considering $n_k(s,o,h)$ the number of visits of the triple $(s,o,h)$ up to episode $k$, and $[k]_{typ}$ the typical episodes for which $n_k(s,o,h)$ is sufficiently large, the following holds true:
    \begin{align}
        \sum_{k=1}^K \mathbb{I}(k \in [k]_{typ}) \sum_{h = 1}^H \frac{1}{n_k(s,o,h)} \leq dSO\ln(Kd) 
    \end{align}
    \begin{proof}
        \begin{align*}
            \sum_{k=1}^K \mathbb{I}(k \in [k]_{typ}) \sum_{h = 1}^H \frac{1}{n_k(s,o,h)} &\myleq{a} \sum_{(s,o) \in S \times O} \sum_{h \in [d]} \sum_{n=1}^{n_K(s,o,h)} \frac{1}{n} \\
            &\myleq{b} dSO \sum_{n=1}^{Kd} \frac{1}{n} \\
            &\myleq{c} dSO\ln(3Kd)
        \end{align*}
    \begin{enumerate}[label=(\alph*)]
        \item Considering $n_k(s,o,h)$ for the whole state space and options space, and considering the summation over $H$ bounded by $d$ elements, for the temporal extension of the actions.
        \item Considering that the maximum number of $(s,o,h)$ visited until episode $K$ is bounded by $Kd$
        \item Considering the rate of divergence of the harmonic series $\sum_{i=1}^n \frac{1}{i} \sim \ln(n)$
    \end{enumerate}
    \end{proof}
\end{lemma}

The following lemmas are adaptations to SMDPs of Lemma 8, 9, and 10 of the paper of the UCBVI paper \cite{azar2017minimax}. We consider to have the same good event $\mathbb{E}$ and $\Omega_{k,h}$.

\begin{lemma}\label{lemma:ucbvi-8}
    Let $k \in [K]$ and $h \in [H]$. Then under the event $\mathbb{E}$ and $\Omega_{k,h}$ of the original paper, the following hold
    \begin{align}
        \sum_{i=1}^k \sum_{j=h}^H \V^{\mu}_{i,j'} \leq KH^2 + 2\sqrt{H^5KL} + 4d^3/3L
    \end{align}
    \begin{proof}
        The proof follows the same passages of the proof of Lemma 8 in \cite{azar2017minimax}, where $j'$ is the next stage after a temporally extended transition.
    \end{proof}
\end{lemma}

\begin{lemma}\label{lemma:ucbvi-9}
    Let $k \in [K]$ and $h \in [H]$. Then under the event $\mathbb{E}$ and $\Omega_{k,h}$ of the original paper, the following hold
    \begin{align}
        \sum_{i=1}^k \sum_{j=h}^H \left( \V^*_{i,j'} - \V^{\mu}_{i,j'}\right) \leq 2HdU_{k} + 4H^2\sqrt{HKL} + 4d^3/3L
    \end{align}
    \begin{proof}
        The proof follows the same passages of the proof of Lemma 9 in \cite{azar2017minimax}, where $j'$ is the next stage after a temporally extended transition.
    \end{proof}
\end{lemma}

\begin{lemma}\label{lemma:ucbvi-10}
    Let $k \in [K]$ and $h \in [H]$. Then under the event $\mathbb{E}$ and $\Omega_{k,h}$ of the original paper, the following hold
    \begin{align}
        \sum_{i=1}^k \sum_{j=h}^H \left( \hat{\V}_{i,j'} - \V^{\mu}_{i,j'}\right) \leq \square Hd U_{k,1} + \square H^2S\square{d^2KLO}
    \end{align}
    \begin{proof}
        The proof follows the same passages of the proof of Lemma 10 in \cite{azar2017minimax}, where $j'$ is the next stage after a temporally extended transition.
        More precisely, what changes is the application of the pigeon hole principle (Lemma \ref{lemma:pigeon}).
    \end{proof}
\end{lemma}
\end{document}